\documentclass[runningheads]{llncs}
\usepackage[T1]{fontenc}
\usepackage{soul}
\usepackage{graphicx}
\usepackage[font=small]{caption}
\usepackage{subcaption}
\usepackage{lineno}
\usepackage{multirow}
\usepackage{tabularx}
\usepackage{xcolor}
\usepackage{booktabs}

\usepackage{makecell}
\usepackage[noend]{algpseudocode}
\usepackage{algorithm}

\usepackage{amsmath}
\usepackage{amsfonts}
\usepackage{amssymb}

\DeclareMathOperator{\PP}{\mathrm{P}}

\newcommand\ceil[1]{\left\lceil#1\right\rceil}
\newcommand\floor[1]{\left\lfloor#1\right\rfloor}
\newcommand\norm[1]{\left\lVert#1\right\rVert}

\newcommand\paren[1]{\left(#1\right)}
\newcommand\curlyb[1]{\left\{#1\right\}}
\newcommand\squareb[1]{\left[#1\right]}

\newcommand\tsc[1]{\text{\textsc{#1}}}

\DeclareMathOperator{\R}{\mathbb{R}}
\DeclareMathOperator{\Z}{\mathbb{Z}}

\DeclareMathOperator{\F}{\mathcal{F}}

\DeclareMathOperator{\VC}{VC}
\DeclareMathOperator{\nn}{NN}

\usepackage{todonotes}

\begin{document}


\title{Towards Practical Finite Sample Bounds for Motion Planning in TAMP}
\titlerunning{Towards Practical Finite Sample Bounds for Motion Planning in TAMP}
%
\author{
    Seiji Shaw, Aidan Curtis, \\Leslie Pack Kaelbling, Tom\'as Lozano-P\'erez, and Nicholas Roy
}

\newcommand{\sas}[1]{\textcolor{teal}{\textbf{SS: #1}}}

\authorrunning{
S. Shaw, A. Curtis, L.P. Kaelbling, T. Lozano-P\'erez, and N. Roy
}
%
\institute{Computer Science and Artificial Intelligence Laboratory, \\Massachusetts Institute of Technology, Cambridge, Massachusetts, USA\\
\texttt{\{seijis,curtisa,lpk,tlp,nickroy\}@csail.mit.edu}
}

\maketitle              
\begin{abstract}
When using sampling-based motion planners such as PRMs, it is difficult to determine how many samples are required for the PRM to find a solution consistently.
This is particularly relevant in Task and Motion Planning (TAMP), where many motion planning problems must be solved sequentially.
We attempt to address this problem by proving an upper bound on the number of samples that are sufficient, with high probability, for a radius PRM to find a \textit{feasible} solution, drawing on prior work in deterministic sampling and sample complexity theory.
We also introduce a numerical algorithm that refines the bound based on the proofs of the sample complexity results we leverage.
Our experiments show that our numerical algorithm is tight up to two to three orders of magnitude on planar problems for radius PRMs but becomes looser as the problem's dimensionality increases.
The numerical algorithm is empirically more useful as a heuristic for estimating the number of samples needed for a KNN PRM in low dimensions.
When deployed to schedule samples for a KNN PRM in a TAMP planner, we also observe improvements in planning time in planar problems.
While our experiments show that much work remains to tighten our bounds, the ideas presented in this paper are a step towards a practical sample bound.

\end{abstract}

\section{Introduction}
Sampling-based motion planners, such as PRMs, perform well in configuration spaces of moderately high dimension.
It is difficult, however, to determine how many samples a PRM needs to find a solution consistently.
This problem is critical when sampling-based motion planners are used for Task and Motion Planning (TAMP), where many motion planning problems must be solved in sequence.
Existing TAMP solvers address this problem by manually selecting a uniform threshold on the motion planner's computational effort on a problem \cite{hauser_integrating_2009,garrett_pddlstream_2020}.
We seek a more principled method of estimating sample thresholds.

We attempt to address our problem by proving a bound on the number of random samples for a PRM to find a solution with high probability. 
First, we must choose a parametrization of the motion planning problem on which we derive our bound.
Our work builds on two problem parametrizations and sample bounds.
Tsao et al. suggest using the widths of passages in the configuration space as a parameterization \cite{tsao_sample_2019}.
They derive the size of a \textit{deterministic} sample set that covers the space at a specified density to yield PRMs guaranteed to have paths in these narrow passages. 
While their parameterization is simple, TAMP systems generally cannot afford to use deterministic sample sets. 
Since the actual widths of passages are unknown \textit{a priori}, many planners return to a previously attempted motion planning problem to refine the graph when a motion plan has not been found on a previous iteration. 
The refined graph must consist of a denser set of deterministic samples that must be recomputed from scratch, wasting previous computation \cite{niederreiter_random_1992,janson_deterministic_2018}.
Hsu et al. introduce a bound for random samples, which can be incrementally refined. 
Their bound uses three parameters that describe the configuration space's `expansiveness,' or visibility properties points and connected sets \cite{hsu1997path}.
This complex parameterization complicates the choice of a procedure to refine the parameters when the planner reattempts a problem.
We wish to have a simple problem parameterization, such as the one provided by Tsao et al., but for random sample sets that can be incrementally constructed in the manner of Hsu et al.

This paper presents a sample bound that extends Tsao et al.'s radius PRM guarantee to random samples.
The core of our argument draws a connection to work in sample complexity by Blumer et al. \cite{blumer_learnability_1989} that estimates the probability of finding covering point configurations of a specific density.
When Blumer et al.'s result is applied in practice, their bound significantly overestimates the number of sufficient samples.
Rather than using the explicit formula stated in the theorem of Blumer et al., we present a numerical algorithm that searches over a tighter expression in their proof \cite{blumer_learnability_1989} to compute a less conservative estimate of the number of samples.
We also extend our radius PRM result to the KNN PRM. 
Our analysis falls short of a non-vacuous bound on the number of sufficient samples, but we include the analysis here for a future finite-sample bound.

Our experiments evaluate the utility of the numerical algorithm in a TAMP planner. 
In radius PRMs, we find that the numerical bounding algorithm is tightest in planar motion planning problems and loosens as the number of dimensions in the configuration space increases.
However, the same numerical algorithm appears to be more effective as a heuristic to guide the sampling effort of a KNN PRM than as a guarantee for the radius PRM itself in low dimensions.
Consequently, the bound also speeds up planning time for planar TAMP problems but shows limited improvement on higher-dimensional problems.
Even so, we believe the theoretical machinery and algorithmic ideas presented in this paper are a step towards a practical bound.

\section{Background and Problem Formulation}\label{sec:problem}
We aim to estimate the sampling effort required to solve a constituent motion planning problem in a task skeleton proposed by a task planner. 
Alg. \ref{alg:tamp_planner} sketches out a general TAMP planning strategy we use in our problem setting, which we walk through in this section.
We use the same PRM definition as Karaman and Frazzoli \cite{karaman_sampling-based_2011}, which we also include in the supplemental material for completeness.

The TAMP planner (Alg. \ref{alg:tamp_planner}) uses a priority queue of task plans or skeletons,  each of which proposes a sequence of motion planning problems to be solved by the PRM (lines \ref{ln:mp-solve-start}-\ref{ln:mp-solve-end}).
If solutions to the motion planning problems are found, they can be combined with the skeleton for a complete solution to the planning problem.
A motion planning problem consists of a configuration space $X \subset \R^d$, a collision-free configuration space $X_{free} \subseteq X$, a start pose $x_s \in X$, and a goal pose $x_g \in X$, which we will denote by the tuple $(X, X_{free}, x_s, x_g)$. 
We assume that $X_{free}$ is Borel measurable and can be sampled uniformly directly, so all our theoretical results and experiments will be done using $X_{free}$ directly without much reference to $X$.
In practice, the measure of $X_{free}$ can be well-approximated efficiently by Monte Carlo integration \cite{boucheron_concentration_2013,vershynin_high_2018}.

\begin{algorithm}[ht]
    \caption{\sc{TampPlanner}}\label{alg:tamp_planner}
    \begin{algorithmic}[1]
        \Require TAMP problem $p$=(start, goal, \dots)
        \State $tp \gets \tsc{TaskPlanner(p)}$
        \For{$(skeleton, progress) \in tp.\tsc{SkeletonQueue()}$}\label{ln:priority-queue}
            \State $success, paths \gets \tsc{True}, []$ 
            \For {$(X, X_{free}, x_s, x_g)$ $\in skeleton$}, \label{ln:mp-solve-start}
               \State $G_{PRM}, \delta, \gamma \gets \tsc{RestoreProgress}(progress)$ \Comment{$\delta$ is passage width.}
               \State $N_{new} \gets \tsc{\textbf{ComputeNewSamples}}(\tsc{vol}(X_{free}), \delta, \gamma)$\label{ln:compute-num-samples}\Comment{$\gamma$ is failure prob.}
               \State $G_{PRM} \gets \tsc{sPrmConstruction}(N_{new}, X_{free}, G_{PRM})$
               \State $path \gets \tsc{sPrmQuery}(x_s, x_g, G_{PRM})$
               \State $progress.\tsc{update}(X_{free}, x_s, x_g, G_{PRM})$
               
               \State \textbf{if} $path \neq \emptyset$ \textbf{then} $paths.\tsc{append}(path)$ \textbf{else}\label{ln:mp-solve-end}
                  
                   \State\;\;\; $\delta',\gamma' \gets \tsc{AdjustWidthAndFailure}(\delta)$\label{ln:shrink-path-width}
                   \State\;\;\; $skeleton.\tsc{UpdateMpProblem}(X_{free}, x_s, x_g, \delta',\gamma')$
                   \State\;\;\; $tp.\tsc{QueuePush}(skeleton)$
                   \State\;\;\; $success \gets \tsc{False}$
                   \State\;\;\; \textbf{break}
          
            \EndFor
            \If{$success$}:
                \State \Return $paths$
            \EndIf
        \EndFor
    \end{algorithmic}
\end{algorithm} 

Whenever sampling-based motion planners are used as a subroutine in TAMP, the planner must decide how many samples to use in each problem (line \ref{ln:compute-num-samples}).
Our main objective is to develop a mathematically well-founded bound on the number of samples required for a PRM to solve the problem (with high probability) and then write an efficient implementation of that bound.
Our bound depends on the width $\delta$ of the passages in $X_{free}$ to find paths and the probability $\gamma$ of failing to find a path even if one exists.
As a result, Alg. \ref{alg:tamp_planner} first computes samples by assuming wide paths on the initial attempt to solve each motion planning subproblem with some initial probability. 
On successive iterations of the same motion planning problem, the algorithm progressively shrinks the path width and decreases the probability of failing to find a path (thus increasing the number of samples; line \ref{ln:shrink-path-width}).
We leave the problem of using the local failure probabilities to compute a global failure probability of the whole planning algorithm as an interesting direction for future work.
In this implementation of Alg. \ref{alg:tamp_planner}, a skeleton containing an infeasible motion planning problem will always remain in the queue (but may not be given a high priority depending on the task planner).
Alternatively, a minimum feasibility threshold can be set using a minimum passage width based on the known limitations of the robot's control infrastructure to follow a trajectory in a narrow corridor.
Should the planner fail to find a motion plan in a skeleton on an iteration using the minimum width, the planner could remove the skeleton from the queue.

The only notation we require is that $B_r^d(x)$ represents a $d$-dimensional ball with radius $r$ centered at point $x$. 
If the location of the ball does not matter (if we are only discussing its volume, etc.), we will abbreviate to $B_r^d$.

\vspace{-0.1cm}
\section{Random Covering Bounds for Radius PRMs}\label{sec:rad_bound}


We begin by first proving a sample bound for radius PRMs.
These PRMs approximate the configuration space with a set of points, where a single point represents a ball region with a radius equivalent to the connection radius.
Tsao et al. observed that if a set of samples forms a sufficiently dense covering of the configuration space, the sample set will be a good approximation of the entire space.
They then showed that roadmaps built from dense covering sets are guaranteed to find paths in narrow passages of a specific width.
In this section, we introduce the terminology to define precisely what we mean by a narrow passage, a covering set of samples, and the relation between the two for the PRM guarantee.
Proofs of all results in this section are in the supplemental material.\footnote{Supplemental material are made available at \texttt{https://sageshoyu.github.io/\\assets/wafr-2024-supp.pdf}.}

\begin{figure}[h!]
    \centering
    \includegraphics[width=0.95\textwidth]{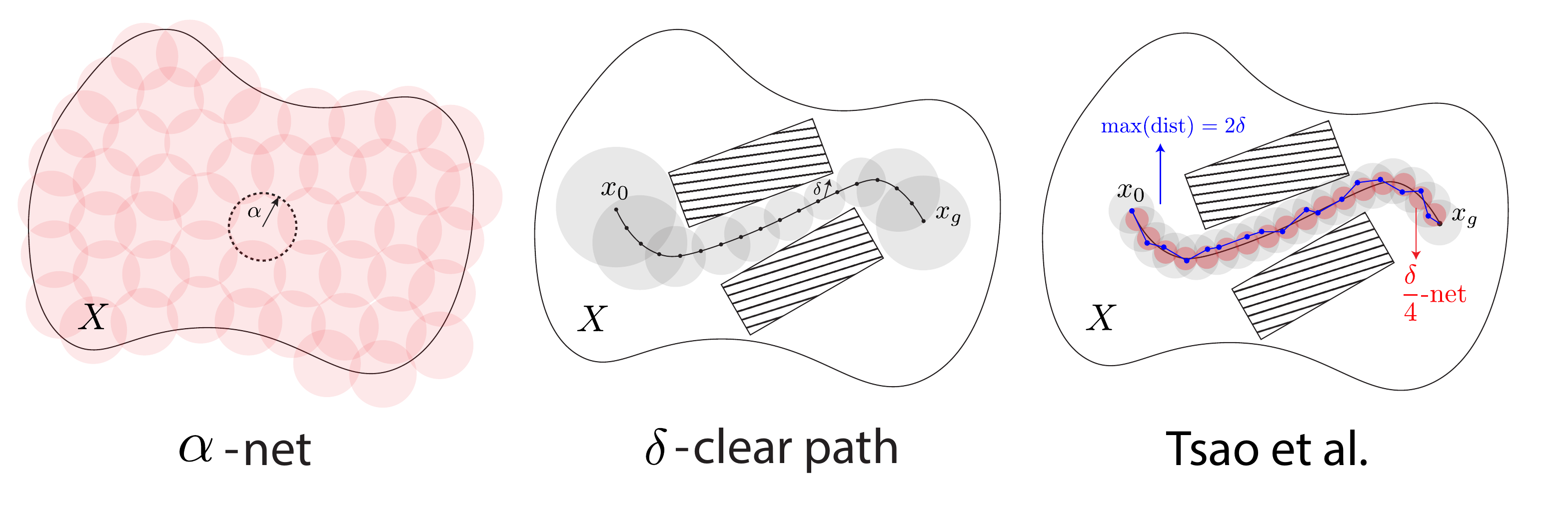}
    
    \caption{\textbf{Left:} An $\alpha$-net, where the points are interpreted as centers of balls, with radius $\alpha$, that collectively cover the rectangle. \textbf{Middle:} An example of a $\delta$-clear path. \textbf{Right:} A depiction of how a covering as dictated by Lemma \ref{lem:coverings_find_passages} yields a roadmap that finds paths in narrow passages.}
    \label{fig:tsao_results}
\end{figure}

We start with the definition of narrow passages.
If there exists a collision-free solution from $x_s$ to $x_g$, then we can associate a `clearance' with the path, the extent to which the robot can deviate from the path without hitting obstacles:
\begin{definition} [Path Clearance]
    Let $(X, X_{free}, x_0, x_g)$ be a motion planning problem. A continuous path $\sigma:[0, 1] \to X$ s.t. $\sigma(0) = x_0$ and $\sigma(1) = x_g$ is \textit{$\delta$-clear} if $B_\delta (\sigma(t)) \subset X_{free}$ for all $0 \leq t \leq 1$. 
\end{definition}
We can imagine sliding a ball with radius $\delta$ by its center along the path.
Should the ball not hit an obstacle, then the path is $\delta$-clear (Fig. \ref{fig:tsao_results}). 
The largest possible $\delta$ characterizes the `width' of the passage through which the path passes. 

The mathematical term for a covering set of points is called an $\alpha$-net. 
A common intuitive definition is first to associate each point with a ball region of radius $\alpha$ centered around the point.
Should the union of the $\alpha$-balls subsume the entire configuration space (`covers' the configuration space), then we call the set of points an $\alpha$-net (Fig. \ref{fig:tsao_results}).
We work with an equivalent definition:
\begin{definition}[$\alpha$-net]\label{def:e-net}
    A finite subset $N \subset X_{free}$ is an $\alpha$-net of $X_{free}$ if for every $x \in X_{free}$ there exists $x_n \in N$ such that $||x - x_n||_2 \leq \alpha$. 
\end{definition}
\noindent We now state a claim proved by Tsao et al. 
A visual depiction of the proof can be found in Fig. \ref{fig:tsao_results}c (and the full proof can be found in the supplemental material).
\begin{lemma}[$\alpha$-nets find paths in $2\alpha$-clear passages]\label{lem:coverings_find_passages}
    Suppose that $N \subset X_{free}$ is a finite set of samples that forms an $\alpha$-net of $X_{free}$, used to construct a radius PRM with connection radius $4\alpha$.
    For all $x_s, x_g \in X_{free}$, if there exists a solution path between $x_s, x_g$ that has clearance $2\alpha$, then the resulting PRM will return a feasible path from $x_s$ to $x_g$.
\end{lemma}

Tsao et al. were interested in deterministic algorithms to construct $\alpha$-nets based on layered point lattices \cite{tsao_sample_2019}.
However, their sample sets are constructed using a fixed net radius $\alpha$ and must be reconstructed from scratch if a finer net is required.
Because TAMP planners like Alg. \ref{alg:tamp_planner} may return to previously initiated motion planning efforts, we wish to increment previously built PRMs to save and reuse computational effort as much as possible. 
In contrast to deterministic sets, collections of samples obtained from random sampling can be naturally incremented when an algorithm revisits a motion planning problem. 

To determine the probability that a set of random samples forms an $\alpha$-net, Blumer et al., in their work on probably approximately correct learning (Theorem 2.1 \cite{blumer_learnability_1989}), provide a bound, which we take as a lemma. 
\begin{lemma}[Number of random samples to form an $\alpha$-net]\label{lem:blumer_e_net_bound}
   Let $X_{free} \subset \R^d$ be a collision-free configuration space, and let $\PP$ be the uniform (Lebesgue) measure on $X_{free}$ and $\gamma \in (0, 1)$ be a failure probability.
   Suppose $N \subset X_{free}$ is a set of $n$ i.i.d. uniform random samples. If: 
   \begin{equation}\label{eqn:blumer_sample_cond}
   n \geq \max \curlyb{
    \frac 4 {\PP(B^d_\alpha)} \log_2 \frac 2 \gamma,
    \frac {8d}{\PP(B^d_\alpha)}  \log_2 \frac {13}{\PP(B^d_\alpha)}
   },
   \end{equation}
   then $N$ is an $\alpha$-net of $X_{free}$ with probability $1 - \gamma$.
\end{lemma}
Putting Lemmas \ref{lem:coverings_find_passages} and \ref{lem:blumer_e_net_bound} together, we now have a bound for radius PRMs.
\begin{theorem}[Radius PRM Guarantee]\label{thm:rad_prm_bound}
    Suppose $N \subset X_{free}$ be a set of $n$ i.i.d. uniform random samples. If: 
    \begin{equation}\label{eqn:rad_prm_sample_cond}
    n \geq \max \curlyb{
        \frac 4 {\PP(B^d_\alpha)} \log_2 \frac 2 \gamma,
        \frac {8d}{\PP(B^d_\alpha)}  \log_2 \frac {13}{\PP(B^d_\alpha)}
    },
    \end{equation}
    then for all $x_s, x_g \in X_{free}$, if there exists a solution path between $x_s, x_g$ that has clearance $2\alpha$, then the a radius PRM with connection radius $4 \alpha$ will return a feasible path from $x_s$ to $x_g$ with probability $1 - \gamma$.
\end{theorem}

Before discussing ways to apply this bound to write \tsc{ComputeNewSamples} in 
Alg. \ref{alg:tamp_planner}, we first want to gain intuition about how path clearance and the dimension of the configuration space influence the number of samples.
To do this, a brief asymptotic analysis is in order.
Using Stirling's approximation, the asymptotics of the volume of a ball with radius $\delta$ as $d \to \infty$ is $\Omega \paren{(\delta/2)^d \cdot (2 \pi e)^{d/2} \cdot d^{-(d+1)/2}}$.
If we fix a constant probability $\gamma$ and desired path clearance $\delta$, we obtain the following asymptotic expression:
\begin{equation}\label{eqn:rad_sample_bound_asymp}
n = O\squareb{
    \frac {C_d} {\delta^d}  \cdot
    \paren{
        \log_2 \frac 1 \gamma 
        + d \log_2 \frac {C_d} {\delta^d}
    }
}, \quad
C_d = \frac{2^d d^{(d+1)/2}}{(2\pi e)^{d/2}}.
\end{equation}
The number of samples varies nearly linearly in $(1/\delta)^d$ and a dimension-dependent constant $C_d$, and logarithmically in $1 / \gamma$.
This result is not surprising; since one must sample enough to cover the entire space in the worst case, the sample complexity must be exponential relative to the dimension of the configuration space.
Our sample complexity argument makes this idea precise and agrees with prior work on narrow passages \cite{hsu_finding_1998,kavraki_analysis_1998}.

We \textit{could} use Theorem \ref{thm:rad_prm_bound} directly by using the right side of the inequality to choose the number of samples for \tsc{ComputeNewSamples} in Alg. \ref{alg:tamp_planner}.
In practice, this bound is quite loose.
The culprit is the maximization operation in Eq. \ref{eqn:blumer_sample_cond} in Lemma \ref{lem:blumer_e_net_bound}.
The maximization's right-hand argument shows the existence of `small' $\alpha$-nets (up to constant factors) by the probabilistic method.
The left-hand argument helps us estimate the probability of failure $\gamma$.
The constant on the right-hand side is relatively large, so $\gamma$ must be small (some hand computations suggest less than $10^{-3}$ for moderate net radii) for the left-hand argument to dominate.
To find a tighter number of samples, we must bypass the maximization operation.
We summarize Blumer et al.'s proof of Lemma \ref{lem:blumer_e_net_bound}  for insight, which estimates the probability that a sample set $N$ does \textit{not} form an $\alpha$-net:
\begin{align}
    &\mathbf P(N \text{ is not a } \alpha\text{-net.})\\
    =&\mathbf P(\exists x \in X_{free} \mbox{ s.t. } N \cap B_\alpha^d = \emptyset) 
    &&\text{Defn. $\alpha$-net, Supp. Material}\\
    \leq& \paren{\sum_{i=1}^{d+1} \binom{2|N|}{i}} \cdot 2^{-\PP(B_\alpha^d) |N|/2} \label{eq:numer_search_exp}
    &&\text{Blumer et al. Thm. A2.1}\\
    <& \paren{\frac {2e|N|} d}^d \cdot 2^{-\PP(B_\alpha^d) |N|/2}
    &&\text{Blumer et al. Prop. A2.1(iii)}
\end{align}
Substituting in the sample inequality in Lemma \ref{lem:blumer_e_net_bound} renders the last line smaller than $\gamma$, for $\gamma$ sufficiently small (Blumer et al. Lemma A2.4).
The third line (Eq. \ref{eq:numer_search_exp}) is a much tighter bound on the probability that $N$ fails to be an $\alpha$-net.
We can now write a numerical algorithm that searches for the smallest $n=|N|$ such that the third line is below a given probability $\gamma$ (Alg. \ref{alg:numerical_bound}).
\newcommand{\ssbound}[1]{\paren{\sum_{i=1}^{d+1} \binom{2#1}{i}} 2^{\frac{-p #1}{2}}}

\begin{algorithm}[H]                \caption{\textsc{NumericalSampleBound}}\label{alg:numerical_bound}
    \begin{algorithmic}[1]
        \Require $1 > \gamma > 0$, $p=\PP(B_\alpha^d)$ volume of $d$-dim. $\alpha$-ball by uniform measure on c-space $X$, $\tsc{SearchExp}(n)=\ssbound{n}$ the tighter exp. from Eq. \ref{eq:numer_search_exp}.
        \State $N_l \gets 1$, $N_u \gets 1$
        \While{$\neg(\tsc{SearchExp}(N_u + 1) < \tsc{SearchExp}(N_u) < \gamma)$}
            \State $N_u \gets 2N_u$ 
            \Comment{Find interval that contains desired sample count.}
        \EndWhile
        \While{$N_l + 1 < N_u$}
            \State $N_{mid} \gets \floor{(N_l + N_u) / 2}$
            \If{$\tsc{SearchExp}(N_{mid} + 1) < \tsc{SearchExp}(N_{mid}) < \gamma$}
                \State $N_u \gets N_{mid}$  \Comment{Binary search to find the tightest sample count.}
            \Else
                \State $N_l \gets N_{mid}$
            \EndIf
        \EndWhile\\
        \Return{$N_u$}
    \end{algorithmic}
\end{algorithm}
\begin{figure}[h]
    \begin{subfigure}{0.55\textwidth}
        \centering
        \includegraphics[width=1.00\textwidth]{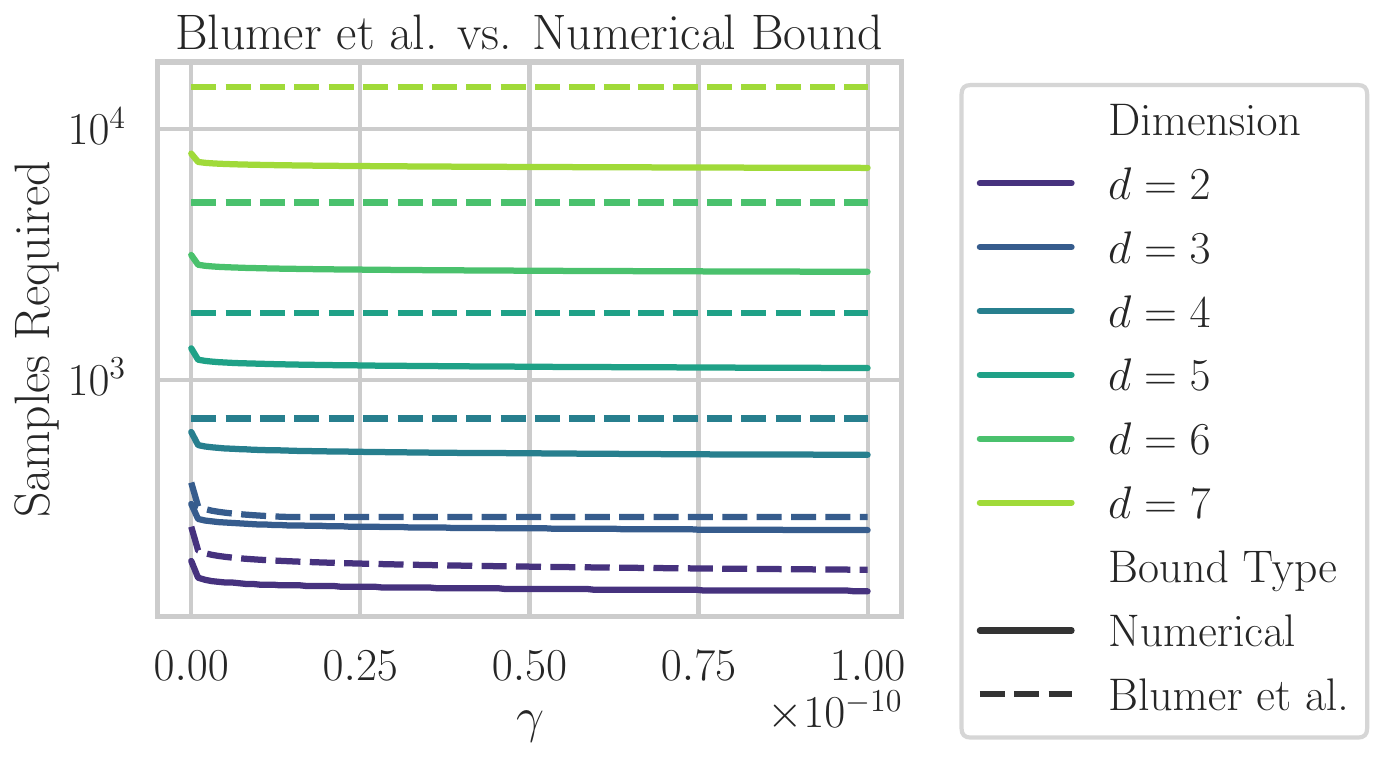}
    \end{subfigure}
    \begin{subfigure}{0.45\textwidth}
        \centering
        \includegraphics[width=1.00\textwidth]{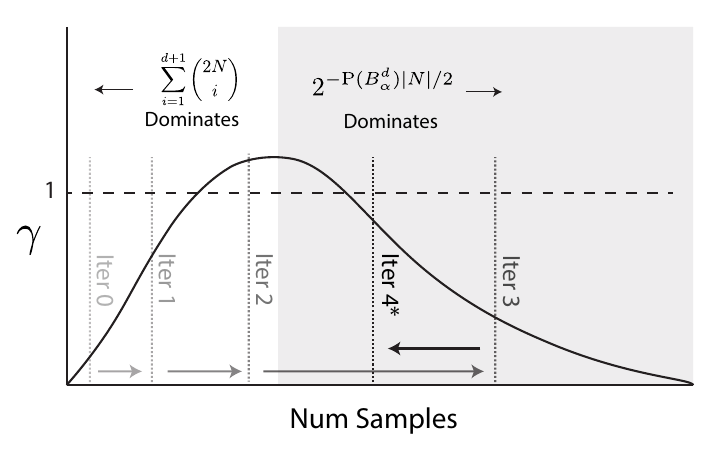}
    \end{subfigure}
    
    \caption{
        \textbf{Left:} A comparison of the sufficient number of samples required to find an $0.5$-net in a unit cube with probability $1 - \gamma$ as computed by Lemma \ref{lem:blumer_e_net_bound} and numerical computation done by Algorithm \ref{alg:numerical_bound}. 
        The right-hand side in the maximization in Lemma \ref{lem:blumer_e_net_bound} dominates for nearly all $\gamma$ plotted, since the dotted lines are nearly horizontal. 
        \textbf{Right:} A depiction of the search conducted by the numerical algorithm. 
    }
    \label{fig:numerical_alg_schematic}
\end{figure}
This algorithm exploits the fact that probability bound in Eq. \ref{eq:numer_search_exp} exhibits a monotonically increasing and then decreasing behavior as a function of the number of samples.
The algorithm uses a doubling scheme to first find a conservative upper bound on the sufficient number of samples on the decreasing side of the global maximum. 
A binary search narrows down the correct sample count (also on the decreasing side of the global maximum, Fig. \ref{fig:numerical_alg_schematic}).
A proof of the algorithm's correctness can be found in the supplemental material.
In practice, we find that using our method results in smaller estimates of sufficient sample counts than using the number of samples predicted directly by Lemma \ref{lem:blumer_e_net_bound}, and allows us to express the scaling relationship of the number of samples with the failure probability $\gamma$ (Fig. \ref{fig:numerical_alg_schematic}).
All comparisons using Eq. \ref{eq:numer_search_exp} are done in log-space to improve the numerics of the procedure.
Our implementation used double-precision floats; arbitrary-precision number representations yielded meager improvements to the computed number of samples.

The runtime of the algorithm is output-sensitive. 
Suppose $n$ is the output number of samples. 
The doubling computation will grow $N_u$ until $N_u \geq n$.
Then, the binary search will search across the integers $[N_u]$.
On every iteration of the doubling and search phases of the algorithm, we sum $d$ binomial coefficient computations.
A binomial coefficient computation $\binom m k$ runs in $O(m)$.
Putting everything together, we can express the runtime as the summation
$O(2 \sum_{i=0}^{\ceil{\log_2 n}} d 2^{i}) = O(dn)$.

At this point, Alg. \ref{alg:numerical_bound} can be used to compute the number of samples in \textsc{ComputeNewSamples}. 
Experiments to verify the tightness of the bound can be found in Section \ref{sec:experiments}.

\section{An Attempt to Extend to KNN PRMs}\label{sec:consequence}

So far, we have focused on radius PRMs.
They are a natural place to start because the connection radius gives us a direct relationship to the clearance of paths we wish our PRMs to find.
In practice, radius PRMs can only be used with relatively few samples since the degrees of the vertices tend to grow very quickly as the number of samples increases \cite{karaman_sampling-based_2011}.
$K$-Nearest Neighbor (KNN) PRMs are often used as a space-efficient alternative. 
This algorithm constructs the roadmap graph by attempting to connect a vertex to its $K$-nearest neighbors in the sample set.

A key assumption in Lemma \ref{lem:coverings_find_passages} is that the connection radius is sufficiently large to connect the $\alpha$-net samples to form the roadmap.
It is tempting to quantify an `effective' connection radius of a KNN PRM, since then our argument to justify Theorem \ref{thm:rad_prm_bound} would follow exactly.
Ultimately, we find this approach leads to a vacuous bound.
However, we include our work here to serve as a good starting point for a finite sample analysis of KNN PRMs.

An effective connection radius of a KNN PRM is a quantity that can be measured \textit{after} sampling and constructing the KNN graph:

\begin{definition}[Connection Radius]
    Let $X_{free}$ be a configuration space, and let $N \subset X_{free}$ be a set of samples.  
    Suppose that $G = (V, E)$ is a $K$-nearest neighbor roadmap graph on $N$. 
    We say $r \in \R^+$ is a \textit{connection radius} of $G$ if $\forall u, v \in V$ if $\norm{u - v} \leq r$, then $u$ is at most the $K$th nearest neighbor of $v$, or $v$ is at most the $K$th nearest neighbor of $u$. 
\end{definition}

There are an infinite number of valid connection radii, and we prefer to find the largest radius we can.
It turns out that we can provide a sharp upper bound on all possible valid connection radii:

\begin{lemma}\label{lem:r_knn_is_well_posed}
    Let $X_{free}$ be a collision-free configuration space, and let $N \subset X_{free}$ be a set of samples. Suppose that $G = (V, E)$ is a $K$-nearest neighbor roadmap graph on $N$.
    Let $r > 0$. Then 
    \begin{equation}\label{eqn:knn_conn_rad_ub}
    r < \min_{v \in V} \norm{v - \nn_{k+1}(v)},
    \end{equation}
    if and only if $r$ is a connection radius of $G$,
    where $\nn_i(v)$ refers to the $i^{th}$ nearest neighbor of vertex $v$.
\end{lemma}
We observe that Eq. \ref{eqn:knn_conn_rad_ub} is a conservative approximation for the true connectivity of the KNN PRM (Fig. \ref{fig:knn_bound_hypo_ratio}a), and thus, so is our conception of the effective connection radius.


The upper bound in Eq. \ref{eqn:knn_conn_rad_ub}, as a function of a random geometric graph of degree $k$, is a random variable, so we must understand its behavior before we treat it as a connection radius.
Fortunately, as the number of samples increases, the realizations of the upper bound tend to concentrate around similar values.
We exploit this fact to prove a lower bound to the expression in Eq. \ref{eqn:knn_conn_rad_ub} to study how the connection radii of a KNN PRM shrinks as the sample count increases.

\begin{proposition}\label{prop:knn_rad_decay}
Let $V$ be a set of $n$ i.i.d. uniform samples from collision-free configuration space $X_{free}$ used to build a $K$-nearest neighbor PRM. Then with probability at least $1 - \gamma$: 

\begin{equation}\label{eqn:knn_conn_rad_lb}
    \squareb{
        \frac{K - \sqrt{2K(\log n - \log \gamma)}}{(n-1)\PP(B_1^d)}
    }^{1/d}
    <
    \min_{v \in V} \norm{v - \nn_{k+1}(v)},
\end{equation}

\end{proposition}
We now use the lower bound stated in Eq. \ref{eqn:knn_conn_rad_lb} to derive a conservative approximation of the effective connection radius using Lemma \ref{lem:r_knn_is_well_posed}:

\begin{corollary}\label{cor:knn_rad_decay}
  Let $V$ be a set of $n$ i.i.d. uniform samples from the collision-free configuration space $X_{free}$ used to build a $K$-nearest neighbor PRM.  
  Let $r \in \R^+$. 
  Then $r$ is a connection radius with probability at least $1 - \gamma$ if
   $$
   r \leq
    \squareb{
        \frac{K - \sqrt{2K\log (n / \gamma)}}{(n-1)\PP(B_1^d)}
    }^{1/d}.
   $$
\end{corollary}
To help parse the bound, we observe the following about the decay rate of the effective connection radius as a function of the failure probability $\gamma$ and number of samples:

\begin{itemize}
    \item As $\gamma$ increases, the effective connection radius grows by $O\paren{\frac{K \log (1 /\gamma)}{\PP(B_1^d)}}^{1/2d}$.
    \item As $n$ increases, the effective connection radius decays by $o\paren{\frac K {n\cdot \PP(B_1^d)}}^{1/d}$.
\end{itemize}
We suspect that the conservative $\sqrt{\log n}$ factor (which makes the bound vacuous as $n \to \infty$ and the reason for the small $o$ above) can be improved with a tighter analysis.
We can even define a similar numerical algorithm to Alg. \ref{alg:numerical_bound} that can search for a tighter \textit{radius} as a function of the number of neighbors, samples, and failure probability. 
For small $n$ (less than 1,000,000), we empirically find that a radius estimated by the numerical algorithm is reasonably tight (numerical experimental results are included in the supplemental material).

Corollary \ref{cor:knn_rad_decay} tells us enough about the effective connection radius for us to apply Lemma \ref{lem:coverings_find_passages}.
It is tempting to combine them with a union bound to make an analogous statement to Theorem \ref{thm:rad_prm_bound} for KNN PRMs:
\begin{theorem}[Vacuous KNN PRM Guarantee]\label{thm:knn_prm_bound}
    Suppose $N \subset X_{free}$ be a set of $n$ i.i.d. uniform random samples, from which we build a $K$-nearest neighbor roadmap. If: 
    \begin{equation}\label{eqn:knn_sample_hypo}
    n \geq \max \curlyb{
        \frac 4 {\PP(B^d_\alpha)} \log_2 \frac 4 \gamma,
        \frac {8d}{\PP(B^d_\alpha)}  \log_2 \frac {13}{\PP(B^d_\alpha)}
    },
    \end{equation}
    \noindent where
    \begin{equation}\label{eqn:knn_rad_hypo}
    \alpha \leq
    \frac 1 4 \cdot
    \squareb{
         \frac{K - \sqrt{2K\log (2n /\gamma)}}{(n-1)\PP(B_1^d)}
    }^{1/d},
    \end{equation}
    
    \noindent then for all $x_s, x_g \in X_{free}$, if there exists a solution path between $x_s, x_g$ that has clearance $2\alpha$, then the PRM will return a feasible path from $x_s$ to $x_g$ with probability $1 - \gamma$.
\end{theorem}

\begin{figure}[ht]
    \centering
    \begin{subfigure}[b]{0.43\textwidth}
        \includegraphics[width=0.95\textwidth]{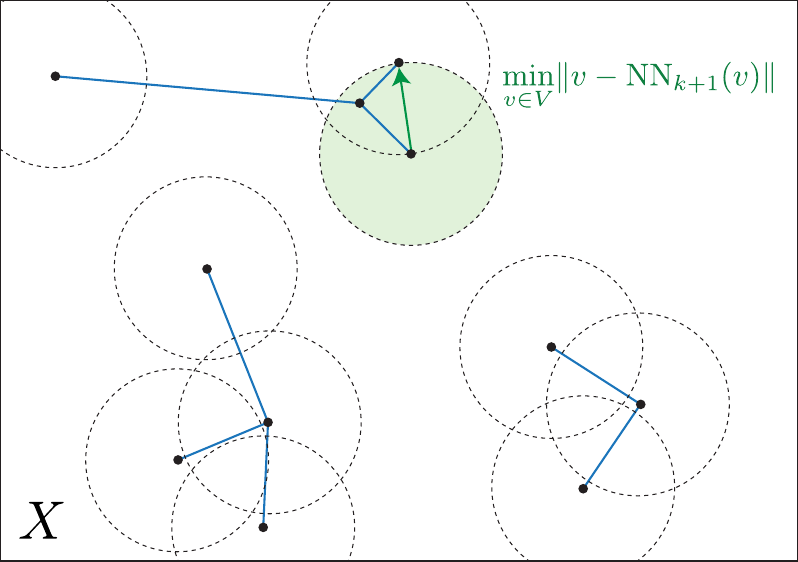}
    \end{subfigure}
    \begin{subfigure}[b]{0.561\textwidth}
        \includegraphics[width=0.90\textwidth]{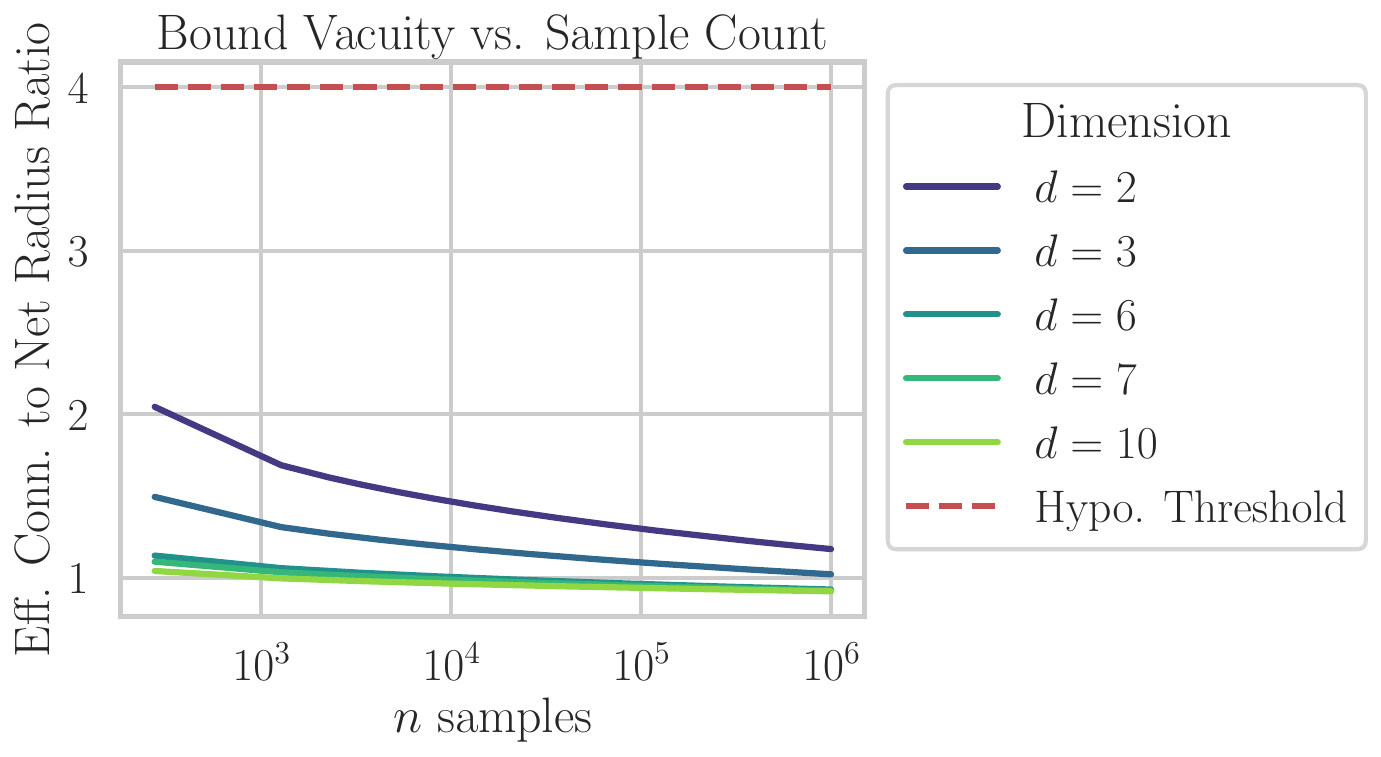}
    \end{subfigure}

    \caption{
        \textbf{Left:} A depiction of the effective connection radius for a PRM graph with $K=1$. Note that we have a conservative approximation that fails to capture `long distance edges' between isolated vertices. \textbf{Right:}
        A plot of the ratio of the connection radius to the net radius as a function of the number of sampled vertices with a fixed failure probability of $\gamma = 0.1$ and $K=256$, as determined by their \textit{numerical procedures}.
        The conditions for Theorem \ref{thm:rad_prm_bound} are satisfied for all sample counts with ratios above 4.0 (dotted red line).
        All curves are below the line, so the conditions of Theorem \ref{thm:knn_prm_bound} are not satisfied for any realistic $n$.
    }
    \label{fig:knn_bound_hypo_ratio}
\end{figure}

\noindent
Equation \ref{eqn:knn_sample_hypo} expresses the same sample condition from Lemma \ref{lem:blumer_e_net_bound}, and Equation \ref{eqn:knn_rad_hypo} fulfills the connection radius condition stated in Lemma \ref{lem:coverings_find_passages}.
Unfortunately, as shown by Fig. \ref{fig:knn_bound_hypo_ratio}, these two conditions can never be satisfied at the same time.
The effective connection radius decays at a rate faster than $o(1/n^{1/d})$, while the largest net radius to be found after $n$ samples is $\Omega(1/n^{1/d})$.
Lemma \ref{lem:coverings_find_passages} stipulates that these two quantities must form a net-radius to the connection-radius ratio of 1:4, which is not true. 
A less conservative estimate of the ratio using the corresponding numerical algorithms for each radius does not meet this requirement either (Fig. \ref{fig:knn_bound_hypo_ratio}).
We suspect the vacuousness of the bound comes from two issues: looseness in the analysis of the effective radius and looseness in the union bound combining Prop. \ref{prop:knn_rad_decay} and Theorem \ref{thm:rad_prm_bound}. 
Improvement in both steps of the argument may result in a non-vacuous sample bound for KNN PRMs.



\section{Experiments}\label{sec:experiments}
We seek to experimentally verify the tightness of Algorithm \ref{alg:numerical_bound} to estimate the number of samples sufficient to find a path. 
In practice, the KNN connection approach is better suited to higher-dimensional problems with many samples.
Since the conditions on Theorem \ref{thm:knn_prm_bound} prevent us from formally bounding the number of samples required by a KNN PRM, we explore using the radius numerical algorithm as a possible sampling \textit{heuristic} to guide the number of samples chosen to build KNN PRMs.

\subsection{Bound Tightness in a Narrow Passageway Problem}
Our clearance experiments are conducted in a `narrow hallway' problem that can be generalized to arbitrary dimensions.
We take the configuration space to be a union of boxes: 
\begin{align*}
&\underbrace{[-1.5, -0.5] \times [-0.5, 0.5]^{d-1}}_{\text{"left end"}} \cup \underbrace{[0.5, 1.5] \times [-0.5, 0.5]^{d-1}}_{\text{"right end"}}
\cup \underbrace{[-0.5, 0.5] \times [-\delta, \delta]^{d-1}}_{\text{hallway}}.
\end{align*}
The hallway is the narrow passage and its center path $[-0.5, 0.5] \times \{0\}^{d-1}$ has clearance $\delta.$
We construct many instances of this hallway problem for varying $\delta < 0.5$ and for $d = \{2, 3, 6, 7, 14, 20\}$. 
We form a Monte Carlo estimate of the probability of finding a path in the narrow passage by running 100 PRMs and checking the query $x_s=(-0.5, 0, \dots, 0), x_g=(0.5, 0, \dots, 0)$ for samples $n=100,1000$.
We then compare the number of samples predicted by the numerical bound (Alg. \ref{alg:numerical_bound}) with the estimated probability.
To explore the use of Alg. \ref{alg:numerical_bound} as a sampling heuristic for KNN PRMs, we compute the bound according to the clearance of the environment and set the probability of failure to $\gamma=0.01$.

\setlength{\tabcolsep}{3pt}
\begin{table}[h!]
\caption{A comparison of the empirical performance of the radius and KNN PRM ($K=32$) and the numerical bound for the hallway problem with $\delta=0.499, 0.25, 0.125$, $d=2,3,4,5,6$, and samples sets with $n=100,1000$. 
The left side of the slash is an MC estimate of the prob. of finding a solution. The right side is the number of samples as output by Alg. \ref{alg:numerical_bound}  using the estimated prob. (minus a machine epsilon to return finite samples) for radius PRMs or $\gamma=0.01$ for the KNN PRM heuristic. If all MC iterations for $n=100$ returned a path, the corresponding box for $n=1000$ is redundant and replaced with a hyphen.
}
\label{tab:mp-results}
\begin{tabular}{|c|c|c|c|c|c|}
    \toprule
    $\delta \downarrow$ & Dim ($d$) & $n=100$ (rad.) & $n=1000$ (rad.) & $n=100$ (KNN) & $n=1000$ (KNN) \\
    \midrule
    \multirow{6}{*}{0.499} & 2 & 1.00/1.42e+05 & -/- & 1.00/1.19e+03 & -/- \\
                           & 3 & 1.00/5.13e+06 & -/- & 1.00/5.20e+03 & -/- \\
                           & 4 & 1.00/2.00e+08 & -/- & 1.00/2.46e+04 & -/- \\
                           & 5 & 1.00/8.20e+09 & -/- & 1.00/1.24e+05 & -/- \\
                           & 6 & 1.00/3.53e+11 & -/- & 1.00/6.60e+05 & -/- \\
    \midrule
    \multirow{6}{*}{0.25}  & 2 & 1.00/1.17e+05 & -/- & 1.00/4.53e+03 & -/- \\
                           & 3 & 0.93/3.19e+06 & 1.00/3.79e+06 & 1.00/3.73e+04 & -/- \\
                           & 4 & 0.37/1.21e+08 & 1.00/1.39e+08 & 1.00/3.45e+05 & -/- \\
                           & 5 & 0.03/5.02e+09 & 1.00/5.57e+09 & 1.00/3.45e+06 & -/- \\
                           & 6 & 0.00/2.19e+11 & 0.83/2.21e+11 & 0.95/3.67e+07 & 0.96/3.67e+07 \\
    \midrule
    \multirow{6}{*}{0.125} & 2 & 0.19/7.37e+04 & 1.00/1.05e+05 & 1.00/1.86e+04 & -/- \\
                           & 3 & 0.00/2.78e+06 & 0.95/2.93e+06 & 0.58/3.24e+05 & 0.99/3.24e+05 \\
                           & 4 & 0.00/1.14e+08 & 0.00/1.14e+08 & 0.09/6.36e+06 & 0.60/6.36e+06 \\
                           & 5 & 0.00/4.87e+09 & 0.00/4.87e+09 & 0.02/1.33e+08 & 0.02/1.33e+08 \\
                           & 6 & 0.00/2.16e+11 & 0.00/2.16e+11 & 0.00/2.89e+09 & 0.00/2.89e+09 \\
    \midrule
    \multirow{6}{*}{0.0625}& 2 & 0.00/6.87e+04 & 0.95/7.53e+04 & 0.86/7.88e+04 & 1.00/7.88e+04 \\
                           & 3 & 0.00/2.71e+06 & 0.00/2.71e+06 & 0.03/2.93e+06 & 0.53/2.93e+06 \\
                           & 4 & 0.00/1.13e+08 & 0.00/1.13e+08 & 0.00/1.19e+08 & 0.00/1.19e+08 \\
                           & 5 & 0.00/4.86e+09 & 0.00/4.86e+09 & 0.00/5.04e+09 & 0.00/5.04e+09 \\
                           & 6 & 0.00/2.15e+11 & 0.00/2.15e+11 & 0.00/2.21e+11 & 0.00/2.21e+11 \\
    \bottomrule
\end{tabular}
\end{table}

Our results can be found in Table \ref{tab:mp-results}.
For radius PRMs, the bound predicts two to three orders of magnitude more than needed for the planar hallway problem and becomes progressively looser as the number of dimensions increases.
This means other conditions besides the sufficient conditions outlined in Lemma \ref{lem:coverings_find_passages} can allow the PRM to find a path through the narrow passage.
In KNN PRMs, the radius bound (as a sampling heuristic) is useful for 1-2 orders of magnitude for low dimensions ($d=2,3$) and loses utility as the number of dimensions grows.
Interestingly, we find that Alg. \ref{alg:numerical_bound} is better suited as a sampling heuristic in lower dimensions for KNN PRMs than as a guarantee for radius PRMs.

We do not find the looseness of the bound surprising.
As the number of dimensions increases, so does the volume of the hallway `ends.'
The bound applies when the entire volume of the configuration space is covered.
Since the hallway ends do not need to be densely sampled to yield a path, the bound becomes increasingly loose as the number of dimensions increases.

\vspace{-0.20cm}
\subsection{TAMP Experiments}
Our experiments in Table \ref{tab:mp-results} suggest that the radius PRM bound may be a useful sampling heuristic for a KNN PRM, which is more tractable to compute in practice.
We now aim to assess the usefulness of the numerical algorithm in the context of a TAMP planner that must solve multiple motion planning subproblems in scenarios where the clearance of the narrow passage is \textit{known} (Sec. \ref{sec:problem} discusses the more general case when the path clearance is unknown).

\begin{figure}[h!]
    \centering
    \begin{subfigure}[b]{\textwidth} 
    \centering
        \includegraphics[width=\linewidth]{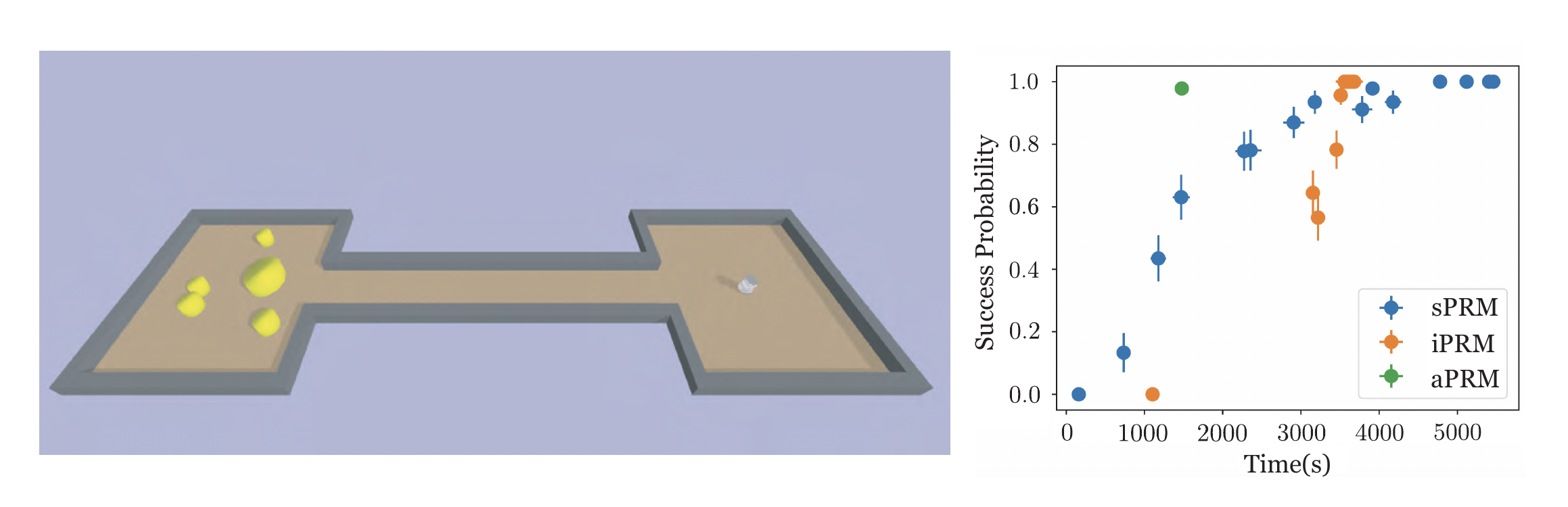}
        \caption{Our 2DoF mobile-base manipulation problem setting (left) and a scatter plot showing TAMP success vs planning time (right) for the three PRM-based motion planning methods. Each point on the scatter plot is the performance and planning time across 100 environment randomizations in which we vary the number and size of objects. The iPRM and sPRM points represent different maximum number of samples $N$.}
        \label{fig:tamp_env}
        \vspace{1em}
         
    \end{subfigure}
    \begin{subfigure}[b]{\textwidth} 
    \centering
        \includegraphics[width=\linewidth]{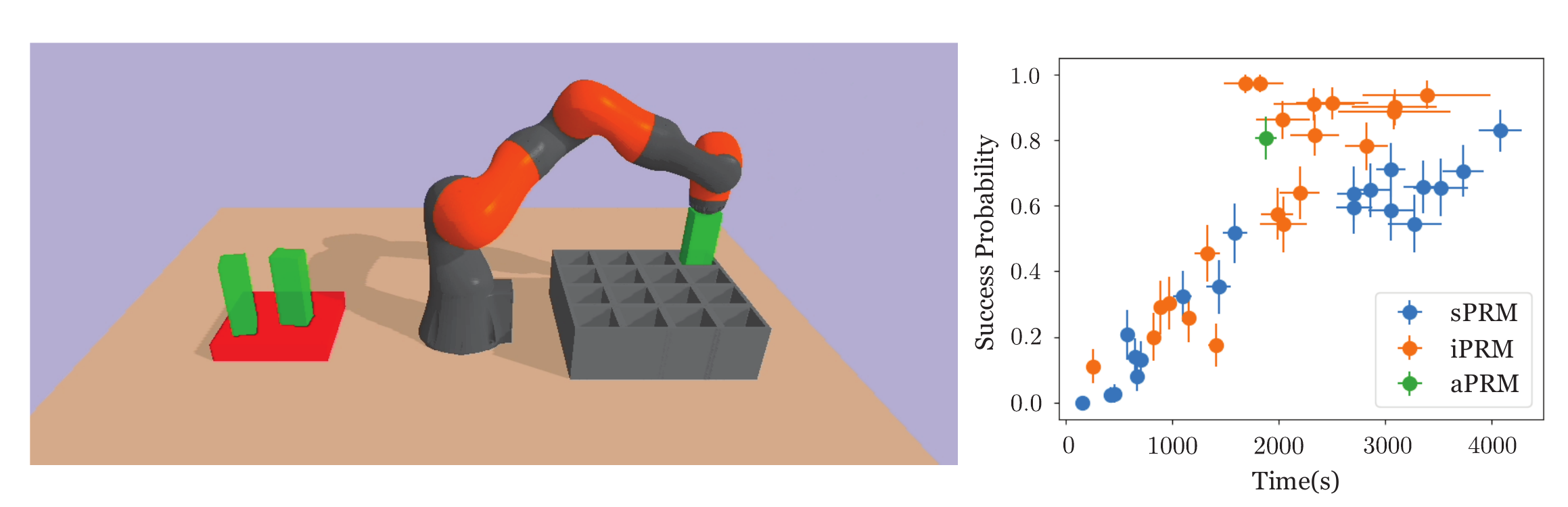}
        \caption{Our 6DoF manipulation problem setting (left) and corresponding baselines (right).}
        \label{fig:robot_env}
    \end{subfigure}
\end{figure}

We examine the bound's performance in a planar task involving a circular robot with 2DoF and a 6DoF peg-in-hole insertion task.
In the task shown in Figure~\ref{fig:tamp_env}, the robot must transfer all yellow objects (5-10) of varying sizes from one room to another by passing through a narrow passage. 
All objects are made to be small enough to fit through the hallway.
In Figure~\ref{fig:robot_env}, a 6DoF KUKA robot arm must move all objects (5-10) of varying size from an initial region into a packing container.
Each robot's planner must solve a sequence of motion planning subproblems of varying difficulty based on the size of the object the robot is holding. 
We use the Focused PDDLStream TAMP algorithm~\cite{garrett_pddlstream_2020}, which proposes plan skeletons consisting of high-level mode sequences (e.g., object grasps and releases) and calls a PRM to refine each candidate task skeleton.

Using an adaptive bound selection strategy, which we will call $\textsc{aPRM}$, we compute an upper bound for the number of samples the KNN PRM uses upon request from the TAMP solver. 
Given the size of the held object, we can compute $\delta$ by subtracting the hallway size from the extent of the robot when holding the object, which can be larger than the extent of the robot itself. 
We then call the KNN version of $\textsc{sPRM}$ outlined by Karaman and Frazzoli \cite{karaman_sampling-based_2011} (see supplementary material) with the number of samples determined by the numerical bound for radius PRMs (Alg. \ref{alg:numerical_bound}) with the computed passage width and probability of failure set to $\gamma = 0.1$. 
The connection radius required by Theorem \ref{thm:rad_prm_bound} is ignored. 

We compare this strategy to two other baseline strategies. 
First, we compare to the $\textsc{sPRM}$ motion planner with a fixed number of samples $N$ selected by hand and fixed across subproblems, which is standard practice in TAMP.
Second, we compare to an incremental version of $\textsc{sPRM}$, which we call $\textsc{iPRM}$. This algorithm iteratively calls $\textsc{sPRM}$ with an increasing max sample bound $K$, which increases geometrically according to $M_{i+1} = cM_{i}$ until a motion plan is found or $M\geq{N}$. Our experiments use $c=1.1$.
We compare these strategies with $\textsc{aPRM}$ in Fig.~\ref{fig:tamp_env} for various selections of $N$. 

In the planar problem, we observe that \textsc{aPRM} is able to solve nearly all problems in less wallclock time (including numerical bound evaluation) than \textsc{sPRM} and \textsc{iPRM}. 
We observed that the bound calculation time was negligible with respect to the PRM planning time.
The radius numerical bound helps the planner effectively trade sampling effort with the complexity of the motion planning problems, whereas \textsc{sPRM} either undersamples or oversamples problems with a single threshold, and \textsc{iPRM} wastes computation time building up to the correct number of samples.
In the 6DoF problem, our results are weaker, where some instances of \textsc{iPRM} scale up to a sufficient number of samples faster than \textsc{aPRM} takes in a single computation.
This is consistent with our observation that the numerical bound becomes less useful as the number of dimensions increases.

\section{Related Work}
Another approach to implicitly allocate samples to a sampling-based motion planner is to write an algorithm to decide the feasibility of the motion planning problem.
A sampling-based motion planner can then be terminated if the motion planning problem is found to be infeasible.
One common approach is to show the existence of a `disconnection proof,' which consists of a hypersurface in the obstacle region that fully encloses the goal from the starting configuration \cite{basch_disconnection_2001,li_sampling_2023,li_scaling_2023}.
Others use volumetric approximations of the configuration space or obstacle region \cite{varava_free_2021,mccarthy_proving_2012,zhang_efficient_2008} to refute feasibility (up to the accuracy of the approximation). 
These methods complement our approach to motion planning in TAMP.
While allocating samples according to a bound (as done in Alg. \ref{alg:numerical_bound}) may help a TAMP planner solve a feasible problem with an appropriate number of samples, the planner can still increase the sample count arbitrarily in an infeasible problem after an unbounded number of attempts.
Developing a suitable criterion to have the planner call an infeasibility checker before adding a large number of samples is an interesting direction for future work.


Much of our understanding of  PRMs with random sampling is from asymptotic results stated in the limit of infinite samples.
Karaman and Frazzoli \cite{karaman_sampling-based_2011} proved that paths returned by PRM (and the more space-frugal PRM*) are both asymptotically optimal.
Tighter estimates on the sufficient connection radius to ensure graph connectivity have been obtained as well \cite{solovey_critical_2020}.
All of these results are proven using arguments that are difficult to extend to finite samples.

In addition to Hsu et al.'s work on configuration space expansiveness \cite{hsu_finding_1998}, there are several instances of finite-random sample bounds for PRMs.
Kavraki et al. \cite{kavraki_analysis_1998}  and Dobson et al. \cite{dobson_geometric_2015} gave a bound on the number of samples required to return a feasible and near-optimal path, respectively, by assuming the clearance and path length.
The stronger assumption on path length allows for simpler arguments that do not rely on covering sets, but also preclude these bounds from application in scenarios where the path length is unknown.

Deterministic sampling schemes, if found, are more sample-efficient than random samples in covering the space under many definitions of density \cite{niederreiter_random_1992}.
Tsao et al. develop a fixed set that meets the dispersion density measure, the radius of the largest unsampled ball in the sample space.
Likewise, for our study, we wish to have an incremental sequence that causes the dispersion to decrease as samples are added.
A nonconstructive argument by Niederreiter \cite{niederreiter_random_1992} shows the existence of an optimally sample-efficient sequence for the dispersion of balls in the $\ell_\infty$-norm (hypercubes), which suboptimally addresses the dispersion in the $\ell_2$-norm related to our study.
To our knowledge, an algorithm that explicitly constructs this sequence has yet to be found.
Finding an optimal deterministic sequence that directly minimizes dispersion in the $\ell_2$-norm is related to the notoriously difficult \textit{sphere-packing problem} \cite{matousek_geometric_1999,conway_sphere_2013}.
Algorithmic constructions of sequences that address other density measures exist \cite{lindemann_incremental_2003,lindemann_incremental_2005} but exploit regular structure available in the other density measures that is unavailable for dispersion.
Empirically, these other sequences are more sample-efficient than random samples in low dimensions \cite{janson_deterministic_2018}, a benefit that becomes negligible in higher dimensions  \cite{hsu_probabilistic_2007}.
For these reasons, we choose to study random sample sequences.




Sample bounds for random coverings have been studied in discrete geometry and computational learning theory \cite{haussler_epsilon-nets_1986,matousek_lectures_2002,mustafa_sampling_2022}. 
The techniques to prove these bounds have been credited to Vapnik and Chervonenkis \cite{vapnik_uniform_1971}, Sauer \cite{sauer_density_1972}, and Shelah \cite{shelah_combinatorial_1972}, who sought to find small coverings by the probabilistic method (see supplementary material).
Blumer et al. build on these techniques to estimate the probability that the samples form a covering \cite{blumer_learnability_1989} to form probably approximately correct learning guarantees.
Their results for the \textit{smallest} $\alpha$-net (right-hand argument in Lemma \ref{lem:blumer_e_net_bound})  were confirmed to be tight up to constants by K\'omlos et al. \cite{komlos_almost_1992}.
This paper connects these results to motion planning.


Lastly, this paper shows the potential usefulness of a sampling bound in a search-then-sample style planner. Search-then-sample is a fairly common approach to 
TAMP~\cite{garrett_pddlstream_2020,sesame,curtis2024partially}. 
These planners use ad hoc sample thresholds for motion planning, which we show can be replaced with a sample bound.
In addition to setting resource limits on individual motion planning subproblems, a sample bound could guide the task-level search as a heuristic to decide which motion planning problems to attempt. 
Some related works along these lines include Cambon et al. \cite{cambon_hybrid_2009}, wherein the number of failed queries is a heuristic to decide which task plans to prioritize for motion computation.

\section{Conclusion}

We present a theoretical bound for the number of samples sufficient for a radius PRM to find a path of a specified clearance with high probability.
Our results are derived by connecting prior analysis of the radius PRM to results from sample complexity theory.
Since the results from sample complexity theory significantly overestimate the number of samples required to find a solution,  we introduce a novel algorithm that computes a tighter number of samples.
We also made progress in extending our theoretical results to the KNN PRM by studying the decay rate of the `effective connection radius' as the number of samples increases.

Our experiments show that despite being less conservative than the original theoretical bound, our numerical sample bound is still too loose for practical use. 
In settings with a single narrow passage using a radius PRM, we saw that the estimated number of samples by the numerical algorithm is tight within two to three orders of magnitude for planar problems.  
The estimate becomes increasingly loose and less useful as the number of dimensions grows.
When the numerical algorithm is applied as a heuristic to estimate the number of samples of the KNN PRM, we found the algorithm to be more useful for low dimension problems ($d=2,3$) than as a guarantee for radius PRMs.
We also applied the heuristic to guide sampling effort of a KNN PRM in a TAMP planner in problems where the narrow passage width was \textit{known}. 
Our results show that only in planar problems, the bound was a useful heuristic to help the planners scale the number of samples appropriately to the sample complexity of the motion planning problem.

In the future, we wish to tighten our theoretical and numerical bounds for more practical use. 
We wish to introduce a version bound that has a more realistic characterization of the worst-case planning problem.
At present, the `worst-case' problem is a configuration space that consists of a \textit{single} narrow passage with constant clearance along every dimension.
In many problems, the clearances vary in the configuration space, with narrow widths in some but not all dimensions.
Introducing additional parameters amenable to iterative search for a more realistic worst case may improve bound tightness for planning problems that appear in practice.
We can further explore Blumer et al.'s proof for improvements: our numerical algorithm only tightens  Blumer's bound over a single lossy step, and there are many more such steps that may lead to a significantly tighter bound.
We also plan to finish extending the sample bound to KNN PRMs to bridge the gap between our theory and experiments. 

\small{\subsubsection{Acknowledgements} We thank David Hsu, Steven LaValle, Sandeep Silwal, Justin Chen, Haike Xu, Josh Engels, Thomas Cohn, Christopher Bradley, Michael Noseworthy, Erick Fuentes, and Chad Kessens for their invaluable support.
This work was supported by the National Science Foundation Graduate Research Fellowship Program under Grant No. 2141064. Any opinions, findings, and conclusions or recommendations expressed in this material are those of the authors and do not necessarily reflect the views of the National Science Foundation.}

%
%
%

\newpage
\bibliographystyle{splncs04}
\bibliography{references}

\pagebreak

\renewcommand{\thesection}{\Alph{section}}
\setcounter{section}{0}
\setcounter{page}{1}

\section*{Supplemental Proofs, Algorithms, and Experiments}
\label{sec:proofs}
\addtocounter{section}{1}
\subsection{PRM Pseudocode}
Karaman and Frazzoli's \cite{karaman_sampling-based_2011} implementation of the PRM algorithm uses two separate subroutines: a pre-processing subroutine that grows a graph to a specified number of samples (Alg. \ref{alg:sprm_construction}), and a query routine that connects the starts and goals to the graph and solves a shortest-path search (using Dijkstra's algorithm, etc.; Alg. \ref{alg:sprm_query}).
For both algorithms, we assume several subroutines have already been provided:
\begin{itemize}
    \item \tsc{SampleFree($\cdot$)}, which samples a collision-free point from the configuration space.
    \item \tsc{CollisionFree($\cdot$)}, which checks if a linear path between the query vertices is collision-free.
\end{itemize}

Standard implementations of the \tsc{Near} subroutine include a range search that identifies the vertices within a radius of the query vertex and a $K$-nearest neighbor search. 

\noindent
\begin{minipage}{0.51\textwidth}
\begin{algorithm}[H]
    \caption{\sc{sPrmConstruction}}\label{alg:sprm_construction}
    \begin{algorithmic}[1]
        \Require $n > 0$ number of new samples, $X$ c-space, $G=(V, E)$ prev. PRM graph
        \State $V$ $\gets V \cup \{\textsc{SampleFree}(X)\}_{i=|V|,\dots,n}$
        \For{$v \in V$} 
            \State $U \gets \textsc{Near}(v, V \backslash \{v\})$
            \For{$u \in U$}
                \If{$\textsc{CollisionFree}(v,u, X)$}
                    \State $E \gets E \cup \{(u,v)\}$
                \EndIf
            \EndFor
        \EndFor
    \State\Return $(V, E)$
    \end{algorithmic}
\end{algorithm}
\end{minipage}
\begin{minipage}{0.49\textwidth}
\begin{algorithm}[H]
    \caption{\sc{sPrmQuery}}\label{alg:sprm_query}
    \begin{algorithmic}[1]
        \Require $x_s$ start, $x_g$ goal, $X$ c-space, $G=(V, E)$ PRM graph
        \State $U_{s}, U_{g} \gets \tsc{Near}(s, V), \tsc{Near}(g, V)$
        \For{$u \in U_s$}
            \If{$\textsc{CollisionFree}(x_s,u, X)$}
                \State $E \gets E \cup \{(x_s,u)\}$
            \EndIf
        \EndFor
        \For{$u \in U_g$}
            \If{$\textsc{CollisionFree}(x_g,u, X)$}
                \State $E \gets E \cup \{(x_g,u)\}$
            \EndIf
        \EndFor
        \State\Return $\tsc{ShortestPath}(G, x_s, x_g)$
    \end{algorithmic}
\end{algorithm}
\end{minipage}

\subsection{Proof of Lemmas \ref{lem:coverings_find_passages} and \ref{lem:blumer_e_net_bound} in Section \ref{sec:rad_bound}}
We begin with our proof of Lemma \ref{lem:coverings_find_passages}, which proceeds along similar reasoning to Tsao et al. \cite{tsao_sample_2019} (though simpler since we do not discuss path optimality).

\begin{proof}[Lemma \ref{lem:coverings_find_passages}]
Let $x_s, x_g \in X_{free}$ be a start and goal pose, respectively. Suppose there exists a continuous $2\alpha$-clear path $\sigma:[0,1] \to X_{free}$ such that $\sigma(0) = x_s$ and $\sigma(1) = x_g$. 

Let $L$ denote the length of $\sigma$. Let $\{B_{2\alpha}^d(p_i)\}_{i \in [\ceil{L/{2\alpha}}]}$ denote a set of balls on $\sigma([0, 1])$, where $B_{2\alpha,1}^d(p_1)$ is centered on $x_s=p_1$ and $B_{2\alpha}^d(p_{\ceil{L/2\alpha}})$ is centered on $x_g=p_{\ceil{L/2\alpha}}$. The remaining balls (for $i = 2, \dots, \ceil{L/{2\alpha}} - 1$) are centered along $\sigma([0, 1])$, where the $p_i$ is spaced $2\alpha(i-1)$ away from $x_s$ along the arclength of $\sigma$.

Next, we construct another set of balls, $\{B^d_{\alpha}(q_j)\}_{j\in[\ceil{L/2\alpha}-1}]$, where $q_j$ is placed $\alpha + 2\alpha (j-1)$ along the arclength of $\sigma$ from $x_s=p_1$. 
We observe that $q_j$ is placed exactly between $p_j$ and $p_{j+1}$ along $\sigma$'s arclength.
Since these points are in a Euclidean space, straight lines represent shortest paths, so
$$
\norm{p_j - q_j}, \norm{p_{j+1} - q_j} \leq 2\alpha,
$$
for all $j \in [\ceil{L/2\alpha}-1]$.
Thus, by construction, we know that 
$$
B^d_{\alpha}(q_j) \subset B^d_{2\alpha}(p_j),B^d_{2\alpha}(p_{j+1}).
$$

Since $N$ forms an $\alpha$-net, we know that there must exist a point $x_j \in N \in B^d_{\alpha}(q_j)$ for all $j$ ($x_j$ must be at most $\alpha$ from ball center $q_j$).
Furthermore, we observe that the linear path between $x_j$ and $x_{j+1}$ is collision-free for all $j$, since $x_j, x_{j+1} \in B^d_{2\alpha}(p_{j+1})$ and balls are convex.

The maximum distance between $x_j, x_{j+1}$ is $4\alpha$ (opposite sides of the $2\alpha$-ball, and so setting the connection radius to $4\alpha$ ensures all $x_s, x_1, \dots, x_{\ceil{L/2\alpha}-1}, x_g$ will be connected to form a path to be returned by the radius PRM (the endpoints trivially included as the centers of the first and last $2\alpha$ balls).\hfill$\blacksquare$
\end{proof}

We now describe the our proof for Lemma \ref{lem:blumer_e_net_bound}. 
Our arguments are centered around covering point configurations and the probability they arise from random sampling processes, a topic well-studied at the intersection of statistics, combinatorics, and discrete geometry\footnote{We will be using nomenclature from both computational/combinatorial geometry and statistics to avoid conflicting terminology.}
\cite{blumer_learnability_1989,komlos_almost_1992,matousek_lectures_2002,mustafa_sampling_2022}.
Our intuition behind the definitions is grounded in trying to `hit' a class of subsets of some space $X_{free}$ with a random finite set of points.

\begin{definition}
   $Y$ be a (potentially infinite) set and let $\F \subset 2^Y$. We call the tuple  $(Y, \F)$ a range space.\footnote{For those who are familiar with learning theory: another name for $\F$ can be the \textit{hypothesis space} of a set of binary classifiers or set of indicators associated with all elements of $\F$.}
\end{definition}

We will take the base space to be $X_{free} \subset \R^n$, the obstacle-free configuration space of a robot.
$\F$ will be the set of $n$-dimensional balls that are a proper subset of $X_{free}$. 
Often, hitting $\textit{all}$ the sets in $\F$ with a finite number of points in $Y$ is too tall an order to ask. 
Only hitting the `voluminous' sets in $\F$ is enough for our purposes.

\begin{definition} \label{def:e-transversal}
    Let $(Y, \F)$ be a range space. Suppose that $\mu$ is a measure on $Y$ such that every $S \in \mathcal F$ is measurable with respect to $\mu$. 
    
    A finite subset $N \subset Y$ is an $\epsilon$-transversal if for all $S \in \mathcal F$ such that $\mu(S) \geq \epsilon \cdot \mu(Y)$, then $N \cap S \neq \emptyset$.
\end{definition}

Intuitively, $\epsilon$ is a \textit{threshold} on the minimum volume of  $S \in \F$ to be hit by a point in $N$. 
There could be many different $\epsilon$-transversals on $(Y, \F)$, and we prefer to find smaller $\epsilon$-transversals if we can. 
Blumer et al. \cite{blumer_learnability_1989}, generalizing Welzl and Haussler's work \cite{haussler_epsilon-nets_1986}  from discrete to continuous range spaces showed that random sampling can find $\epsilon$-transversals with high probability.
We can derive the following statement by combining Theorem A2.1, Proposition A2.1, and Lemma A2.4 by Blumer et al. \cite{blumer_learnability_1989}:

\begin{theorem}\label{thm:blumer_e_net_bound}
   Let $(Y, \mathcal F)$ be a range space with VC-dimension $\VC(\F)$ and  $Y \subset \R^n$. 
   Let $\PP$ be a probability measure on $Y$ such that all $S \in \mathcal F$ is measurable with respect to $\PP$. 
   
   Let $N \subset Y$ be a set of $m$ independent random samples of $Y$ with respect to $\PP$, where
   $$
   m \geq \max \curlyb{
    \frac 4 \epsilon \log_2 \frac 2 \gamma,
    \frac {8d} \epsilon \log_2 \frac {13} \epsilon
   },
   $$
    
    \noindent then $N$ is an $\epsilon$-transversal of $Y$ with probability at least $1-\gamma$.
\end{theorem}

Intuitively, the VC-dimension quantifies the \textit{complexity}, or intricacy, of how the sets in $\F$ intersect with $Y$ and each other.
We now move to the technical definition of the VC-dimension:
\begin{definition}
    Let $(Y, \F)$ be a range space. We say that a subset $A \subset Y$ is \textit{shattered} by $\F$ if each of the subsets of $A$ can be obtained as the intersection of some $S \in \F$ with $A$. 
    We define the \textit{VC-dimension} of $\F$ as the supremum of the sizes of all finite shattered subsets of $Y$.
    If arbitrarily large subsets can be shattered, the VC-dimension is $\infty$.
\end{definition}

\noindent
Our argument only requires that the VC-dimension of a set of $d$-dimensional (Euclidean) balls in $\R^d$ is $d+1$ \cite{matousek_lectures_2002,mustafa_sampling_2022}. 
When uniformly sampling from a space $X_{free}$, we are in a favorable scenario where the definitions of $\epsilon$-transversal (hitting set of balls in $X_{free}$)  $\epsilon$-net (covering $X_{free}$ with balls) defined in Def. \ref{def:e-net} coincide (but for different $\epsilon$'s).

\begin{lemma}\label{lem:geom_e-net_iff_comb_e-net}
    Let $Y \subseteq \R^d$, and $\PP$ be a uniform probability measure. 
    Let $N \subset Y$ be a finite subset. 
    
    $N$ is an $\epsilon$-net of $Y$ if and only if $N$ is an  $\epsilon^d\PP(B^d_1)$-transversal on the range space $(Y, \mathcal F)$ where $\mathcal F$ is the set of $d$-dimensional balls in $\mathbb R^d$ that have non-empty intersection with $Y$.
\end{lemma}
\begin{proof}

    The essence of this proof is a radius-volume conversion, which is only possible since we are working with range spaces with spheres.
   
   We start with the forward direction. 
   Suppose $N$ is an $\epsilon$-net. 
   Let $S \in \mathcal F$ such that $\PP(S) \geq \epsilon^d \PP(B_1^d) \cdot \PP(Y) = \epsilon^d \PP(B_1^d)$ (i.e., $S$'s volume is larger than that of a $d$-dimensional ball with radius $\epsilon$). 
   Since $S$ is a ball, its radius must be larger than $\epsilon$. Let $x_S$ be its center. 
   By the definition of geometric $\epsilon$-net, there exists $x_n \in N$ such that $||x_c - x_n|| \leq \epsilon$, so $x_n \in S$.

   The backward direction is even shorter. Let $N$ be a $\epsilon^d \PP(B_1^d)$-transversal. 
   Let $x \in Y$. Let $S$ be the $d$-dimensional ball centered at $x$. 
   Then there must exist $x_n \in N$ such that $x_n \in S$, so $||x - x_n|| \leq \epsilon$.\hfill$\blacksquare$
\end{proof}
\noindent
We now have all the ingredients we require for Lemma \ref{lem:blumer_e_net_bound}.

\begin{proof}[Lemma \ref{lem:blumer_e_net_bound}]
    By Lemma \ref{lem:geom_e-net_iff_comb_e-net}, $\epsilon^d\PP(B^d(1))$-transversals is an $\epsilon$-net.
    We apply Lemma \ref{thm:blumer_e_net_bound}, and the result follows.
    \hfill$\blacksquare$
\end{proof}

\subsection{Proofs of Lemma \ref{lem:r_knn_is_well_posed} and Prop. \ref{prop:knn_rad_decay}}
The proof of Lemma \ref{lem:r_knn_is_well_posed} is a simple proof by contradiction. 
The intuition for the proof is pictorially displayed in Fig. \ref{fig:knn_bound_hypo_ratio}.
 
\begin{proof}[Lemma \ref{lem:r_knn_is_well_posed}]
    We start with the forward direction. 
    Let $r \in \R$ such that $0 < r < \min_{v \in V} \norm{v - \nn_{k+1}(v)}$. 
    Suppose, for the sake of contradiction, that there exists $u, v \in V$ such that \textsc{CollisionFree}$(u, v)$ returns \textsc{True} but $(u, v) \notin E$.
    Then $v$ must not be a $K$th nearest neighbor of $u$, or vice versa. Without loss of generality, assume the former, and so $v$ must be at least a $(K + 1)$th nearest neighbor of $u$.
    But then $r < \min_{v \in V} \norm{v - \nn_{k+1}(v)} < \norm{u - v}$, so we have a contradiction. 

    The proof of reverse direction nearly follows the definition of an effective connection radius.
    Suppose $r \geq \min_{v \in V} || v - \nn_{k+1}(v)||$. 
    Let $v^*$ be the minimizing vertex in the expression above. 
    Then $(v^*, \nn_{k+1}(v))$ is a non-edge and $\norm{v^* - \nn_{k+1}(v^*)} \leq r$.\hfill$\blacksquare$
\end{proof}

Our proof of Prop. \ref{prop:knn_rad_decay} is loosely inspired by the proof for Lemma 58 presented by Karaman and Frazzoli \cite{karaman_sampling-based_2011}. 

\begin{proof}[Proposition \ref{prop:knn_rad_decay}]
We prove our bound by controlling the number of samples that land within a ball of radius $r$  of another sample using a Chernoff bound.

Let $r = \min_{v \in V} \norm{v - \nn_{k+1}(v)}$.
Let $I_{v, w}$ be a Bernoulli random variable where vertex $w$ falls within radius $r$ for vertex $v$. We aim to find a lower bound to the probability the event:
\begin{equation}\label{eq:knn-rad-event}
\PP\paren{\forall v \in V, \sum_{w \in V \backslash \{v\}} I_{v, w}(r) < K + 1}
\end{equation}
\noindent
It is far easier to reason about an upper bound of the complement of this event and take a union bound over all vertices:
\begin{align*}
&\PP\paren{\exists v \in V, \sum_{w \in V \backslash \{v\}} I_{v, w}(r) \geq K + 1}\\
\leq 
&\sum_{v \in V} \PP\paren{\sum_{w \in V \backslash \{v\}} I_{v, w}(r) \geq K + 1}\\
\leq
&N \cdot \PP\paren{v \in V \text{ s.t. } B_r^d(v) \subset X, \sum_{w \in V \backslash \{v\}} I_{v, w}(r) \geq K + 1}
\end{align*}
The second line takes the union bound over all vertices.
The third line applies the bound $\PP(B_r^d(v) \cap X) \leq \PP(B_r^d(v))$.
To obtain a tail estimate on the binomial random variable above, we use a Chernoff bound stated by Boucheron et al. \cite{boucheron_concentration_2013}:
\begin{align*}
&\PP\paren{v \in V \text{ s.t. } B_r^d(v) \subset X, \sum_{w \in V \backslash \{v\}} I_{v, w}(r) \geq K + 1}\\
\leq
&\exp \curlyb{-D_{KL}(p + t || p) (N-1)}
\end{align*}
where $D_{KL}$ is the KL-divergence between two Bernoulli variables, and $p = \PP(B_r^d)$ and $t = K / (n-1) - p$ (our bound on $r$ will avoid a degenerate $t \leq 0$).
Using the well-known inequality $D_{KL}(p_1 || p_2) \geq (p_1 - p_2)^2 / 2p_1$ when $p_1 > p_2$:
\begin{align*}
\exp \curlyb{-D_{KL}(p + t || p) (N-1)}
&\leq \exp\curlyb{\frac{-(N-1)t^2}{2(p + t)}}\\
&\leq \exp\curlyb{
    - (N-1)\frac{
        \paren{\frac K {N - 1} - p}^2
    }{
        2 \paren{\frac K {N-1}}
    }
}\\
&= \exp\curlyb{
    - \frac{
        (K - p(N-1))^2
    }{
        2K 
    }
}
\end{align*}
Substituting the bound above back into our union bound, plugging in $p = \PP(B_r^d)$, imposing that the last expression be less than $\gamma > 0$, and solving for $r$ yields:.
$$
\PP(B_r^d) = p \leq \frac{K - \sqrt{2K(\log N - \log \gamma)}}{N-1}
$$

\begin{equation}\label{eqn:sufficeint_conn_upperbound}
r \leq \squareb{
    \frac{K - \sqrt{2K(\log N - \log \gamma)}}{(N-1)\PP(B_1^d)}
}^{1/d}
\end{equation}
What we have shown is that if we choose a ball radius $r$ as expressed in Equation \ref{eqn:sufficeint_conn_upperbound}, then each ball of radius $r$ centered at each vertex in the PRM contains less than $K + 1$ other vertices with probability at least $1 - \gamma$.
By the construction of the event in Eq. \ref{eq:knn-rad-event}, we know that:

$$
\squareb{
    \frac{K - \sqrt{2K(\log N - \log \gamma)}}{(N-1)\PP(B_1^d)}
}^{1/d}
< 
\min_{v \in V} \norm{v - \nn_{k+1}(v)},
$$
with probability $1 - \gamma$.\hfill$\blacksquare$
\end{proof}
\noindent
The conservative $\log N$ factor inside the radical may be removed with a tighter analysis that does not rely on a union bound over all vertices.

\subsection{Numerical Algorithm Correctness Proof and Additional Results}
\subsubsection{Analysis of Correctness of Algorithm \ref{alg:numerical_bound}}

We begin by stating the following theorem that formalizes the correctness of Algorithm \ref{alg:numerical_bound} and then sketch out the proof:

\begin{theorem}
    Let $f: \Z^+ \to \R^+$ denote the function represents the Eq. \ref{eq:numer_search_exp}.
    $$f(m) = \paren{\sum_{i=1}^{d'} \binom{2m}{i}} \cdot 2^{-cm}$$
    for $c > 0$, and $d' \in \Z^+$, which are computed by the input to Alg. \ref{alg:numerical_bound}.

    Let the input failure probability be $\gamma \in (0, 1)$. Then Alg. \ref{alg:numerical_bound} will return $m^*$ such that $f(m^*-1) > \gamma > f(m^*)$.
\end{theorem}
\begin{proof}
    By the definition of the binomial coefficient operation, we know that the left-hand term is a $d'$-degree polynomial that is monotonically increasing in $m$ with all positive coefficients:
    $$
    \sum_{i=1}^{d'} \binom{2m}{i} = \sum_{i=0}^{d'} a_i m^i, \mbox{ } a_i \geq 0 \forall i=0, \dots, d'
    $$
\noindent
We then compute the derivative of $f$. Using the product rule, we see that:
    $$
    f'(m) = \paren{\sum_{i=1}^{d'} ia_{i-1}m^{i-1}} \cdot 2^{-cm} - cm \paren{\sum_{i=0}^{d'} a_{i}m^{i}} \cdot (2^{-cm} \log 2)
    $$
    We observe from the form of the expression (lower order polynomial and exponential on the left term, higher order polynomial and linear factor on the right term) and conclude two possibilities:
    
     \begin{enumerate}
        \item $f'(m) < 0$ for all $m \geq 1$  ($f$ is monotonically decreasing), or
        \item $f'(m) \geq 0$ for some $k > m$, and then $f(m) < 0$ for all $m > k$ ($f$ is monotonically increasing until $k$, and then starts monotonically decreasing).
    \end{enumerate}

    \noindent
    If we are in the first case, we write $k = 0$ for the convenience of the rest of the analysis.
    In either case, the doubling search of the numerical algorithm will double $N_u$ until $N_u > k$ and $f(N_u) \leq \gamma$.
    Thus, we observe that the binary search will preserve two invariants:
    
     \begin{enumerate}
        \item $N_l \leq m^* \leq N_u$. The left inequality is true because the search increases $N_l$ if either $f(\floor{(N_l + N_u) / 2}) > \gamma$ or $N_l \leq k$ (by the check that verifies $f$ is decreasing via the left inequality on Line 6). The right inequality is true via the right inequality on Line 6.
        \item The intervals close by half each time.
    \end{enumerate}
    
    \noindent
    Thus, the algorithm will return the correct answer since we are searching over a discrete set of integers while the invariants hold.
\end{proof}

\subsubsection{Numerical Algorithm for Proposition \ref{prop:knn_rad_decay} and Tightness Experiments}

As mentioned in Section \ref{sec:consequence}, we write an analogous algorithm (Alg. \ref{alg:numerical_radius_bound}) to compute the \textit{largest} effective connection radius given a failure of probability $\gamma$ (and other technical details of the configuration space).
The algorithm performs a binary search up to a prespecified additive approximation error.
The correctness of the algorithm follows from a similar argument provided for Alg. \ref{alg:numerical_bound} above.

\begin{algorithm}[H]                \caption{\textsc{NumericalRadiusBound}}\label{alg:numerical_radius_bound}
    \begin{algorithmic}[1]
        \Require $1 > \gamma > 0$, $\PP(B_1^d)$ volume of unit ball by uniform measure on c-space $X$, $K$ number of neighbors, $n$ number of samples, $r_u$ maximum possible radius in $X$, and maximum $\epsilon > 0$ additive error of solution.
        \State $r_l \gets \epsilon$
        \State $p \gets K / (n - 1)$
        \While{$r_u - r_l > \epsilon$}
            \State $r_{mid} \gets (r_l + r_u) / 2$
            \State $q_{mid} \gets r_{mid}^d \PP(B_1^d)$
            \State $q_{mid-} \gets (r_{mid} - \epsilon / 2)^d \PP(B_1^d)$
            \Comment{Ensure probability bound is decreasing.}
            \If{$e^{-D_{KL}(p || q_{mid-})} \geq \exp^{-D_{KL}(p || q_{mid})}$ \text{ or } $(n-1)\exp^{-D_{KL}(p || q_{mid})} > \gamma$}
                \State $r_u \gets r_{mid}$  \Comment{Binary search to ensure radius corresponding to $\gamma$ is in $[r_l, r_u]$.}
            \Else
                \State $r_l \gets r_{mid}$
            \EndIf
        \EndWhile\\
        \Return{$r_u$}
    \end{algorithmic}
\end{algorithm}
\vspace{-0.5cm}

To verify the tightness of Prop. \ref{prop:knn_rad_decay} and its numerical algorithm \ref{alg:numerical_radius_bound}, we sample uniformly from a unit cube, run an approximate $K$-nearest neighbors algorithm, and compute the effective radius upper bound for varying numbers of neighbors, samples, and dimensions.
The experiment is repeated for a variety of dimensions and nearest-neighbor counts.
We see that the bound is overly conservative with small sample sets and then progressively tightens as samples increase (with a maximum number of samples of $n = 10^6$).
For $d=2,3$ the numerical bound tightens quickly but does not fully converge for $d = 20$ within $10^6$ samples.
Eventually, as $n \to \infty$, we expect the bound to become loose again because of the conservative logarithmic factor.
\begin{figure}
    \centering
    \includegraphics[width=\textwidth]{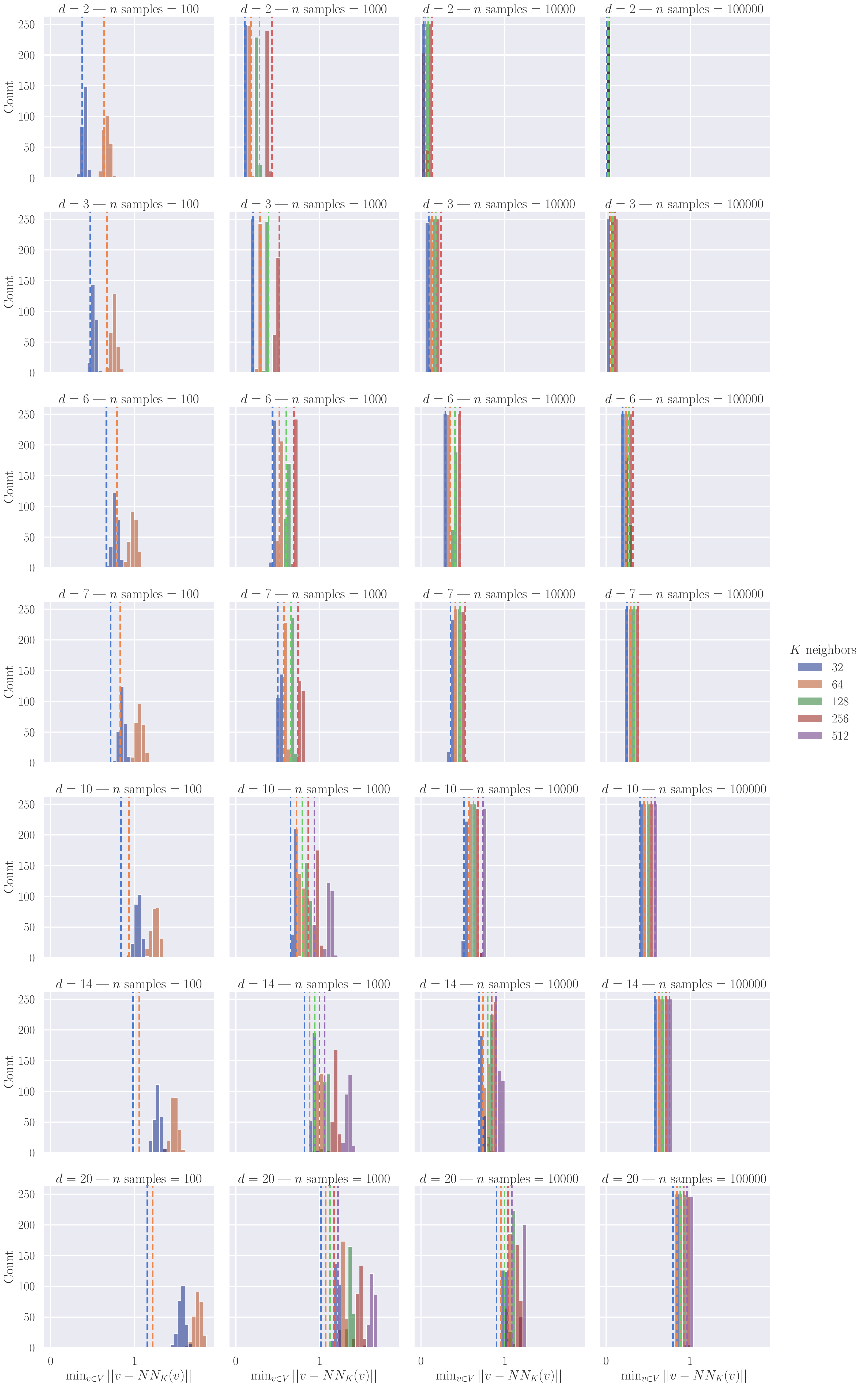}
    \caption{The largest connection radius computed by the numerical bound by setting the probability of failure $\gamma = 0.01$ is represented by corresponding dotted vertical bars. Their right-to-left ordering is the same as the ordering of the histograms.} 
    \label{fig:knn_empirical}
\end{figure}

\subsection{Code Implementation}

The code that implements all the algorithms and experiments described in this paper is available in a public repository: \texttt{https://github.com/robustrobotics/nonasymptotic-mp}.

\end{document}